\documentclass{article}
\usepackage{ijcai21}

\usepackage{times}
\usepackage{soul}
\usepackage{url}
\usepackage[hidelinks]{hyperref}
\usepackage[utf8]{inputenc}
\usepackage[small]{caption}
\usepackage{graphicx}
\usepackage{amsmath}
\usepackage{amsthm}
\usepackage{booktabs}
\usepackage{algorithm}
\usepackage[all]{abaisero}

\usepackage[detect-weight=true,detect-inline-weight=math,separate-uncertainty=true]{siunitx}
\usepackage{etoolbox}
\robustify\bfseries

\usepackage[capitalize]{cleveref}
\usepackage{bm}
\usepackage{xfrac}
\usepackage{verbatim}
\usepackage[inline]{enumitem}
\usepackage{todonotes}

\usepackage{subcaption}

\usepackage{multirow}
\usepackage[noend]{algpseudocode}

\usepackage{tikz}
\usetikzlibrary{backgrounds,decorations.pathreplacing}

\newcommand\ones{\bar{1}}

\newcommand\iset{\mathcal{I}}
\newcommand\qset{\mathcal{Q}}
\newcommand\xset{\mathcal{X}}
\newcommand\zset{\mathcal{Z}}

\newcommand\nohist{\varepsilon}
\newcommand\notest{\varepsilon}

\newcommand\pomdp{^{(b)}}
\newcommand\psr{^{(p)}}
\newcommand\rpsr{^{(r)}}

\newcommand\core{^\dagger}

\newtheorem{lemma}{Lemma}
\newtheorem{theorem}{Theorem}
\newtheorem{corollary}{Corollary}
\newtheorem{proposition}{Proposition}

\title{Reconciling Rewards with Predictive State Representations
\footnote{Code available at \url{https://github.com/abaisero/rl-rpsr}.}}

\author{
Andrea Baisero
\and
Christopher Amato\\
\affiliations
Northeastern University, Boston, Massachusetts, USA\\
\emails
\{baisero.a, c.amato\}@northeastern.edu
}

\begin{document}

\maketitle

\begin{abstract}
  Predictive state representations (PSRs) are models of controlled non-Markov
  observation sequences which exhibit the same generative process governing
  POMDP observations without relying on an underlying latent state.  In that
  respect, a PSR is indistinguishable from the corresponding POMDP\@.  However,
  PSRs notoriously ignore the notion of rewards, which undermines the general
  utility of PSR models for control, planning, or reinforcement learning.
  Therefore, we describe a sufficient and necessary \emph{accuracy} condition
  which determines whether a PSR is able to accurately model POMDP rewards, we
  show that rewards can be approximated even when the accuracy condition is not
  satisfied, and we find that a non-trivial number of POMDPs taken from a
  well-known third-party repository do not satisfy the accuracy condition.  
  We propose \emph{reward-predictive state representations} (R-PSRs), a
  generalization of PSRs which accurately models both observations and rewards,
  and develop value iteration for R-PSRs.  We show that there is a mismatch
  between optimal POMDP policies and the optimal PSR policies derived from
  approximate rewards.  On the other hand, optimal R-PSR policies perfectly
  match optimal POMDP policies, reconfirming R-PSRs as accurate state-less
  generative models of observations and rewards.
\end{abstract}

\section{Introduction}

Predictive state representations (PSRs) are models of controlled observation
sequences which exhibit the same generative properties as partially observable
Markov decision processes
(POMDPs)~\cite{littman_predictive_2002,singh_predictive_2004}.  Compared to
POMDP models, PSRs lack an underlying latent state;  instead, the system state
is grounded in predicted likelihoods of future observations.  The structure of
PSRs only involves observable quantities, therefore learning a PSR model is
generally considered to be simpler than learning a latent-state model such as a
POMDP~\cite{wolfe_learning_2005}.
Hence, significant research effort has been collectively spent on PSR model
learning~\cite{singh_learning_2003,james_learning_2004,rosencrantz_learning_2004,wolfe_learning_2005,wiewiora_learning_2005,bowling_learning_2006,mccracken_online_2006,boots_closing_2011}.
Likewise, a number of control and reinforcement learning methods have been
successfully adapted to PSRs, e.g., policy
iteration~\cite{izadi_planning_2003}, incremental
pruning~\cite{james_planning_2004}, point-based value
iteration~\cite{james_improving_2006,izadi_point-based_2008}, and
Q-learning~\cite{james_planning_2004}.

While PSRs are a promising alternative to POMDPs for modeling observation
sequences, they notoriously lack the ability to model rewards appropriately.
In fact, we show that PSRs are able to represent only a specific subset of
reward functions representable by finite POMDPs.  
Prior work in PSR-based control can be split in two groups depending on how
reward modeling is addressed:
\begin{enumerate*}[label=(\alph*)]
  \item In the first group, a linear PSR reward function which encodes the
    appropriate task is directly given and/or assumed to
    exist~\cite{izadi_planning_2003,izadi_point-based_2008,boots_closing_2011},
    which means that these methods cannot be used to solve problems which
    cannot be modeled by vanilla PSRs.  Further, directly designing PSR rewards
    which are not grounded to a latent state is extremely unintuitive, compared
    to designing POMDP rewards grounded on state.
  \item In the second group, a \emph{reward-aware} PSR (RA-PSR) models rewards
    by making them explicitly observable and an integral part of the agent's
    observable history~\cite{james_planning_2004,james_improving_2006}.
    However, in partially observable sequential decision making problems such
    as POMDPs, rewards are a construct meant to rank agent behaviors based on
    how well they solve a task, rather than an intrinsic aspect of the
    environment like states and observations.  In fact, it is not uncommon for
    rewards to be available during offline training but not during online
    execution.  Further, an agent which is able to observe rewards will be able
    to condition its behavior based on rewards and/or infer the latent state
    using rewards, which is not generally allowed when solving POMDPs.
    Therefore, while RA-PSRs are able to represent reward functions which are
    not representable by vanilla PSRs, they are still unable to represent the
    full range of control tasks representable by POMDPs, which do not assume
    observable rewards.
    %
  %
\end{enumerate*}

In this work, we first develop the theory of PSR reward processes, and then a
novel extension of PSRs which is able to model any reward process without
making rewards explicitly observable by the agent.
Our contributions are as follows:
\begin{enumerate*}[label=(\alph*)]
  \item We derive a sufficient and necessary \emph{accuracy} condition which
    determines whether the rewards of a POMDP can be accurately represented by
    a PSR, and a linear approximation for when the accuracy condition is not
    satisfied;
  \item We develop \emph{reward-predictive state representations} (R-PSRs), a
    generalization of PSRs capable of representing the reward structure of any
    finite POMDP;
  \item We adapt Value Iteration (VI) to R-PSRs;
  \item We show that a non-trivial portion of common finite POMDPs, taken from
    a third-party repository, do not satisfy the accuracy condition; and
  \item We evaluate the optimal policies derived by POMDPs, PSRs and R-PSRs
    according to the other's reward structure, confirming that even the best
    linear reward approximations in PSRs may catastrophically alter the
    underlying task, while the reward structure of R-PSRs perfectly matches
    that of the original POMDPs. 
\end{enumerate*}
Our work represents a significant step towards being able to use PSRs to model
realistic environments which do not assume observable rewards, and opens the
door for more complicated and scalable methods able to exploit the structure of
the proposed R-PSR model.


\section{Background}


\paragraph{Notation} 
We use bracket notation $\left[ v \right]_i$ and $\left[ M \right]_{ij}$ to
index vectors and matrices.  $\left[ M \right]_{i:}$ and $\left[ M
\right]_{:j}$ respectively indicate the i-th row vector and the j-th column
vector of $M$.  We use $\Delta \xset$ to indicate the set of probability
distributions over set $\xset$, and boldface $\bm x$ to indicate a random
variable.
To avoid ambiguity, we will sometimes denote quantities associated with POMDPs
using superscript $\pomdp$, those associated with PSRs using $\psr$, and those
associated with R-PSRs using $\rpsr$.

\subsection{POMDPs}

A POMDP~\cite{cassandra_acting_1994} is a tuple $\langle \sset, \aset, \oset,
\tfn, \ofn, \rfn, \gamma \rangle$ constisting of: state, action, and
observation spaces $\sset$, $\aset$, $\oset$; transition function
$\tfn\colon\sset\times\aset\to\Delta\sset$; observation function
$\ofn\colon\sset\times\aset\times\sset\to\Delta\oset$; reward function
$\rfn\colon\sset\times\aset\to\realset$; and discount factor $\gamma\in\left[0,
1\right]$.
We focus on finite POMDPs, whose sets $\sset$, $\aset$ and $\oset$ are finite;
consequently, POMDP functions and related quantities can be represented as
matrices and vectors, e.g., the reward matrix $\left[ R \right]_{ij} \doteq
R(s=i, a=j)$.

\paragraph{Interactions, Histories, and Beliefs}
We define an \emph{interaction} $ao\in\aset\times\oset$ as an
action-observation pair representing a single exchange between agent and
environment, and its \emph{generative} matrix
$G_{ao}\in\realset^{|\sset|\times|\sset|}$ as $\left[ G_{ao} \right]_{ij}
\doteq \Pr(s'=i, o \mid s=j, a)$, which encodes both state transition and
observation emission probabilities.
A \emph{history} $h\equiv a_1o_1a_2o_2 \ldots$ is a sequence of interactions,
and the cumulative observable past.  We use \emph{string concatenation} to
denote the concatenation of histories and/or interactions, e.g.,  $h_1h_2$ or
$hao$.  We denote the space of all histories as
$\hset\doteq\left(\aset\times\oset\right)\kstar$, and the empty history as
$\nohist$.
A \emph{belief} $b\colon\hset\to\Delta\sset$ is the distribution over states
following history $h$, i.e., vector $\left[ b(h) \right]_i \doteq \Pr(s = i
\mid h)$.  We define the history reward function $R(h, a) \mapsto \Exp\left[
R(\bm s, a) \mid h \right] = b(h)\T \left[ R \right]_{:a}$.

\subsection{Predictive State Representations}

A PSR is a discrete-time controlled dynamical system which lacks the notion of
a latent state~\cite{littman_predictive_2002,singh_predictive_2004};  rather,
the system state is composed of the predictive probabilities of hypothetical
futures called \emph{tests}.

\paragraph{Tests and their Probabilities}
A \emph{test} $q\equiv a_1o_1a_2o_2\ldots$ is---like a history---a sequence of
interactions.  We denote the space of all tests as
$\qset\doteq\left(\aset\times\oset\right)\kstar$, and the empty test as
$\notest$.  Although histories and tests are structurally equivalent, they
differ semantically in that the former refer to the past, and the latter to
hypothetical futures.
We generalize the \emph{generative} matrix to tests via $G_q = \cdots
G_{a_2o_2} G_{a_1o_1}$.
The action and observation sequences associated with a test $q$ are denoted
respectively as $a_q \in\aset\kstar$ and $o_q\in\oset\kstar$.
A \emph{test probability} $p(q\mid h) \doteq \Pr(o_q \mid h, a_q)$ is the
probability of the test observations $o_q$ if the test actions $a_q$ are taken
from the history $h$ as a starting point.  Test probabilities are the core
quantity modeled by a PSR\@.  A \emph{linear} predictive state $p(h)
\in\realset^{|\qset\core|}$ (where $\qset\core$ is a \emph{core} set of tests,
defined later) is a representation of history $h$ such that test probabilities
are linear in $p(h)$, i.e., $p(q\mid h) = p(h)\T m_q$, where $m_q
\in\realset^{|\qset\core|}$ is the \emph{parameter vector} associated to test
$q$.
Littman and Sutton~\shortcite{littman_predictive_2002} show the vectorized form
of test probabilities based on beliefs (using a less general version of the
generative matrix $G_q$),
\begin{equation}
  p(q\mid h) = b(h)\T G_q\T \ones \,, \label{eq:psr}
\end{equation}
\noindent where the linear products represent an implicit marginalization over
the sequence of latent states.

\paragraph{Outcome Vectors}
The \emph{outcome} of a test $u(q)\in \left[0, 1\right]^{|\sset|}$ is a vector
indicating the test probabilities from each state as a starting point, i.e.,
$\left[ u(q) \right]_i \doteq \Pr(o_q \mid s=i, a_q)$.  Outcomes can be defined
recursively via the generative matrix,
\begin{align}
  u(\notest) &= \ones \label{eq:psr.outcome.init} \,, \\
  u(aoq) &= G_{ao}\T u(q) \,. \label{eq:psr.outcome.recursive}
\end{align}
\noindent Combining \Cref{eq:psr,eq:psr.outcome.init,eq:psr.outcome.recursive}
results in $p(q\mid h) = b(h)\T u(q)$, i.e., the test probability given a
history is the expectation of test probabilities given each state.
A set of tests is said to be \emph{linearly independent} iff the respective
outcome vectors are linearly independent, and any maximal set of linearly
independent tests is called a \emph{core set}, denoted as $\qset\core$.  While
there are infinite core sets, they all share the same size $|\qset\core|$, called
the PSR \emph{rank}, which is upper-bounded by $|\qset\core| \leq
|\sset|$~\cite{littman_predictive_2002}.

\paragraph{Predictive States}
The outcome matrix $U\in\left[0, 1\right]^{|\sset|\times|\qset\core|}$ of a core
set $\qset\core$ is the column-wise stacking of the core outcome vectors $\{
u(q) \mid q\in\qset\core \}$.  
By the definition of a core set, the outcome $u(q)$ of any non-empty test
$q\neq \notest$ is a linear combination of the core outcome matrix $U$ columns
(else the core set would not be maximal), i.e., $u(q) \in\colspace U$ and,
because $UU\PI$ is the projection onto $\colspace U$, then $u(q) = UU\PI u(q)$.
Consequently,
$p(q\mid h) = b(h) \T u(q) = b(h)\T UU\PI u(q) = p(h)\T m_q$,
where $p(h)\T \doteq b(h)\T U$ is the predictive state, and $m_q \doteq U\PI
u(q)$ is the parameter vector of $q$.  Each dimension of $p(h)$ is itself the
test probability of a core test, i.e., $\left[ p(h) \right]_i = p(q_i\mid h)$,
where $q_i$ is the $i^{th}$ core test.

\paragraph{Emissions and Dynamics}
Observation probabilities are $\Pr(o\mid h, a) = p(ao\mid h) = p(h)\T m_{ao}$,
while the predictive state dynamics are $p(hao) = \sfrac{ \left( p(h)\T M_{ao}
  \right) }{ \left( p(h)\T m_{ao} \right) }$, where $M_{ao}$ is the column-wise
  stacking of the extended core test parameters $\{ m_{aoq} \mid q\in\qset\core
  \}$.

\subsection{Value Iteration}\label{sec:vi}

Ultimately, we wish to compare POMDP, PSR and R-PSR models by comparing their
respective induced optimal policies.  Although a number of modern solution
methods could be used for this purpose, we employ simple value iteration (VI)
as an arguably necessary stepping stone, leaving more modern methods for future
work.
Value iteration (VI) is a family of dynamic programming algorithms which
estimate the optimal value function $V\opt{}\iter{k}(h)$ for increasing
horizons $k$, from which optimal actions can be derived.  Variants have been
developed for POMDPs and PSRs~\cite{james_planning_2004,boots_closing_2011}.  A
$k$-horizon policy tree $\pi$ is composed of an initial action $a_\pi$ and a
$(k-1)$-horizon sub-policy tree $\pi_o$ for each possible observation.  We
denote the space of $k$-horizon policies as $\Pi\iter{k}$.

Value iteration for POMDPs (POMDP-VI)~\cite{cassandra_acting_1994} is based on
the linearity of policy tree value functions in the belief state, $V_\pi(h) =
b(h)\T \alpha\pomdp_\pi$, where $\alpha\pomdp_\pi$ is the \emph{alpha vector}
representing the values of $\pi$.  The optimal $k$-horizon value function is
the maximum over policy value functions, $V\opt{}\iter{k}(h) =
\max_{\pi\in\Pi\iter{k}} b(h)\T \alpha\pomdp_\pi$, and is notably piecewise
linear and convex (PWLC)~\cite{smallwood_optimal_1973}.  POMDP-VI iteratively
computes the alpha vectors $\alpha_\pi$ of policy trees $\pi\in\Pi\iter{k}$ for
increasing horizons using
$\alpha_\pi = \left[ R\pomdp \right]_{:a_\pi}$ if $k=1$, and $\alpha_\pi =
\left[ R\pomdp \right]_{:a_\pi} + \gamma \sum_o G_{ao}\T \alpha_{\pi_o}$ if
$k>1$.
%
%
In practice, some alpha vectors are dominated by others, and can be pruned to
mitigate the exponential growth~\cite{cassandra_incremental_2013}.

Value iteration can be adapted to PSRs
(PSR-VI)~\cite{james_planning_2004,boots_closing_2011} under a linear PSR
reward function $R\psr(h, a) \doteq p(h)\T \left[ R\psr \right]_{:a}$.  Many
properties of POMDP-VI remain valid for PSR-VI, including the linearity of
value functions $V\psr_{\pi} = p(h)\T \alpha\psr_\pi$, and the PWLC property of
the optimal $k$-horizon value function.  PSR-VI iteratively computes the alpha
vectors $\alpha\psr_\pi$ of policy trees $\pi\in\Pi\iter{k}$ for increasing
horizons using
$\alpha_\pi = \left[ R\psr \right]_{:a_\pi}$ if $k=1$, and $\alpha_\pi = \left[
R\psr \right]_{:a_\pi} + \gamma \sum_o M_{ao} \alpha_{\pi_o}$ if $k>1$.
%

\section{The Failure of PSRs as Reward Models}

Given any finite POMDP, the respective PSR state $p(h)$ holds sufficient
information to represent the probability of future observations.  In this
section, we show that the same cannot be said about being able to represent
future rewards.  We develop theory regarding a PSR model's (in)ability to
accurately model POMDP rewards.

\subsection{Theory of PSR Reward Accuracy}

\begin{proposition}\label{thm:psr.failure}
  For any finite POMDP and its respective PSR, a (linear or non-linear)
  function $f(p(h), a) \mapsto R\pomdp(h, a)$ does not necessarily exist.
  %
  (proof in Appendix).
\end{proposition}

For tractability reasons, we consider the family of linear PSR rewards
represented by matrix $R\psr\in\realset^{|\qset\core|\times|\aset|}$, such that
$R\psr(h, a) \doteq p(h)\T \left[ R\psr \right]_{:a}$.
Ideally, it would always be possible to express the true rewards $R\pomdp(h,
a)$ in such a way.  Unfortunately, this is not always possible.

\begin{theorem}[Accurate Linear PSR Rewards]\label{thm:psr.rewards}
  A POMDP reward matrix $R\pomdp$ can be accurately converted to a PSR reward
  matrix $R\psr$ iff every column of $R\pomdp$ is linearly dependent on the
  core outcome vectors (the columns of $U$).  If this condition is satisfied,
  we say that the PSR is \emph{accurate}, and $R\psr = U\PI R\pomdp$
  accurately represents the POMDP rewards.
  %
  (proof in Appendix).
\end{theorem}

As a direct consequence of \Cref{thm:psr.rewards}, a PSR can only accurately
model a sub-space of POMDP rewards $R\pomdp \in \{ UW \mid W \in
\realset^{|\qset\core|\times|\aset|} \}$.  Also, full-rank PSRs are always
accurate, while low-rank PSRs---often praised for their representational
efficiency---are most likely to suffer from this issue.
\Cref{thm:pomdp.rewards} shows the reverse problem does not exist.

\begin{corollary}\label{thm:pomdp.rewards}
  Assuming that a PSR can be represented by a finite POMDP, then
  any PSR rewards $R\psr$ are accurately represented by POMDP rewards $R\pomdp
  = UR\psr$.
  %
  (proof in Appendix).
\end{corollary}

Next, we consider whether it is possible to formulate an appropriate
approximation of POMDP rewards for a non-accurate PSR.
Ideally, we care for PSR rewards which induce the optimal policy and/or values
most similar to those of the POMDP\@.  However, it is extremely challenging to
fully analyze the effects of rewards on policies (which is the control problem
itself); therefore, we will use reward errors as a proxy.
Note that, while this methodology is extremely simple, it is also imperfect,
since small reward errors may lead to large policy errors, while large reward
errors may lead to small policy errors.
While it is possible to consider preferences over histories/beliefs which
should result in a lower approximation error (e.g., the reachable
histories/beliefs), that kind of prior knowledge is not common.  Therefore, we
consider reward approximations where the approximation errors for all
histories/beliefs are weighted uniformly.

\begin{theorem}[Approximate Linear PSR Rewards]\label{thm:psr.rewards.approx}
  The linear approximation of POMDP rewards for non-accurate PSRs which results
  in the lowest reward approximation error is $R\psr \doteq U\PI R\pomdp$.
  %
  (proof in Appendix).
\end{theorem}

Notably, \Cref{thm:psr.rewards,thm:psr.rewards.approx} share the same
expression for $R\psr$, which is unsurprising since an accurate estimate is
just an errorless approximate estimate.
The ability to compute approximate rewards via \Cref{thm:psr.rewards.approx}
begs the question of whether they are accurate enough to serve as a basis for
control.  Unfortunately, the connection between PSR rewards and optimal policy
is not as straightforward and interpretable as in POMDPs, since the rewards are
encoded in terms of test probabilities rather than states.
To answer this question quantitatively, one can solve both the POMDP and the
approximate PSR, and compare the respective optimal policies/values directly;
however, this approach is very expensive.  Luckily,
\Cref{thm:psr.reconstruction} provides a simple qualitative approach which
allows one to interpret approximate PSR rewards directly by
\emph{reconstructing} equivalent POMDP rewards.

\begin{corollary}\label{thm:psr.reconstruction}
  $\tilde R\pomdp \doteq UU\PI R\pomdp$ is the reconstructed POMDP-form of the
  PSR approximation $R\psr$ of the true POMDP rewards $R\pomdp$.  $\tilde
  R\pomdp$ and $R\pomdp$ are equal iff the accuracy condition is satisfied.
  %
  (proof in Appendix).
\end{corollary}
Next, we provide a detailed case study demonstrating the theory developed here.
In \Cref{sec:evaluation} we show empirically that, if there is any
approximation error at all, then there is a high chance that the policy has
been altered catastrophically.

\subsection{A Case Study of Approximate PSR Rewards}\label{sec:casestudy}

\begin{figure}
  \centering
  \resizebox{\linewidth}{!}{%
    \begin{tikzpicture}

  \tikzset{node/.style={draw, thick, circle}}
  \tikzset{edge/.style={thick, ->, shorten >=2pt}}
  \tikzset{edge2/.style={edge, <->, shorten <=2pt}}

  \node[node] (0) {$0$};
  \node[node, below=of 0] (1) {$1$};

  \node[node, right=of 0] (2) {$2$};
  \node[node, below=of 2] (3) {$3$};

  \node[node, right=of 2] (4) {$4$};
  \node[node, below=of 4] (5) {$5$};

  \node[node, right=of 4] (6) {$6$};
  \node[node, below=of 6] (7) {$7$};

  \node[node, right=of 6] (8) {$8$};
  \node[node, below=of 8] (9) {$9$};

  \draw[edge] (0) to[out=10, in=170] node[anchor=south] {$0 \left( \frac{1}{2} \right)$} (2);
  \draw[edge] (2) to[out=190, in=-10] (0);
  \draw[edge2] (2) -- (4);
  \draw[edge2] (3) -- (5);
  \draw[edge2] (4) -- (6);
  \draw[edge2] (5) -- (7);
  \draw[edge] (7) to[out=10, in=170] (9);
  \draw[edge] (9) to[out=190, in=-10] node[anchor=north] {$0 \left( \frac{1}{2} \right)$} (7);

  \draw[edge] (1) -- node[anchor=west] {$1 \left( \frac{1}{2} \right)$} (2);
  \draw[edge] (6) -- (8);
  \draw[edge] (8) to[out=0, in=10] node[anchor=east] {$1 \left( \frac{1}{2} \right)$} (9);
  \draw[edge] (9) to[out=-10, in=-90, loop] node[anchor=west] {$0 \left( \frac{1}{2} \right)$} (9);

  \draw[edge] (0) to[out=170, in=90, loop] node[anchor=east] {$0 \left( \frac{1}{2} \right)$} (0);
  \draw[edge] (1) to[out=180, in=190] node[anchor=west] {$1 \left( \frac{1}{2} \right)$} (0);
  \draw[edge] (3) -- (1);
  \draw[edge] (8) -- node[anchor=east] {$1 \left( \frac{1}{2} \right)$} (7);

  \node[left=of 0] {loaded};
  \node[left=of 1] {unloaded};
  \draw[thick, decorate,decoration={brace, amplitude=3pt, raise=1.1cm}] (0.south) -- (0.north);
  \draw[thick, decorate,decoration={brace, amplitude=3pt, raise=1.1cm}] (1.south) -- (1.north);

  \node[below=of 1] {loading};
  \node[below=of 5] {travel};
  \node[below=of 9] {unloading};
  \draw[thick, decorate,decoration={brace, amplitude=3pt, raise=1.1cm}] (1.east) -- (1.west);
  \draw[thick, decorate,decoration={brace, amplitude=4pt, raise=1.1cm}] (7.east) -- (3.west);
  \draw[thick, decorate,decoration={brace, amplitude=3pt, raise=1.1cm}] (9.east) -- (9.west);

\end{tikzpicture}
  }

  \caption{Load/unload domain.  Rows indicate the agent's status (\emph{loaded} or
    \emph{unloaded}), and columns indicate the agent's position (observed as
    \emph{loaded}, \emph{travel}, and \emph{unloaded}).  Movements (\emph{left}
    and \emph{right}) are deterministic, and non-zero rewards are shown as ``$x
    \left( y \right)$'', where $x$ is the POMDP reward, and $y$ is the PSR
  approximation.  }\label{fig:loadunload}
\end{figure}
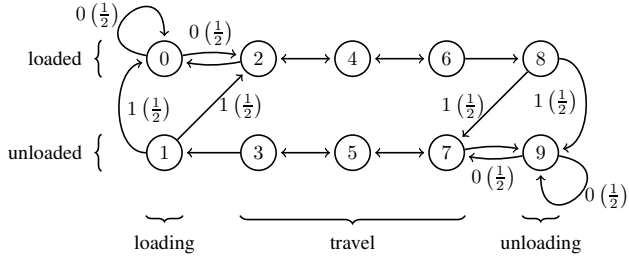

We use the \emph{load/unload} domain (POMDP in Appendix) shown in
\Cref{fig:loadunload} as a case study to show a catastrophic failure of
approximate PSR rewards.  The agent navigates a corridor of $5$ cells under
partial observability of its own position and whether it is carrying a load or
not.  Rewards are given for loading (when not loaded) at one end of the
corridor, and unloading (when loaded) at the other end.  The task is to keep
moving back and forth between one end of the corridor (loading) and the other
(unloading).

\begin{table}
  \centering
  \begin{tabular}{ll}
    \toprule
    Core test $q\in\qset\core$ & Outcome $u(q)\T$ \\
    \midrule
    left loading & $(1\; 1\; 1\; 1\; 0\; 0\; 0\; 0\; 0\; 0)$ \\
    right travel & $(1\; 1\; 1\; 1\; 1\; 1\; 0\; 0\; 0\; 0)$ \\
    right unloading & $(0\; 0\; 0\; 0\; 0\; 0\; 1\; 1\; 1\; 1)$ \\
    right travel, left loading & $(1\; 1\; 0\; 0\; 0\; 0\; 0\; 0\; 0\; 0)$ \\
    left travel, right travel  & $(0\; 0\; 0\; 0\; 1\; 1\; 1\; 1\; 0\; 0)$ \\
    \bottomrule
  \end{tabular}
  \caption{Core tests $\qset\core$ and respective outcome vectors for
  load/unload, found using a breadth-first variant of the search algorithm by
Littman and Sutton~\protect\shortcite{littman_predictive_2002}.}
  \label{tab:loadunload:core}
\end{table}

The domain has $|\sset|=10$ states, but its PSR rank is $|\qset\core|=5$.
\Cref{tab:loadunload:core} shows a core set $\qset\core$ and the respective
outcome vectors $u(q)$.  By \Cref{thm:psr.rewards} there is an entire subspace
of reward functions which cannot be accurately represented, and there is a
chance that the corresponding PSR model is not accurate;  we will verify that
this is indeed the case.
We note that the outcomes in \Cref{tab:loadunload:core} are also the columns of
matrix $U$, and that every pair of rows in $U$ (columns in
\Cref{tab:loadunload:core}) is identical.  This means that the PSR is unable to
make any distinction between states $0$ and $1$, $2$ and $3$, etc, which will
cause a problem, because the POMDP reward function specifically needs to
differentiate between states $0$ and $1$, and $8$ and $9$.

Next, we show the POMDP rewards $R\pomdp$ (note that only states $1$ and $8$
emit rewards),
\begin{equation}
  \resizebox{.89\linewidth}{!}{$
  \displaystyle
  R\pomdp = \begin{pmatrix}
    0.0 & 1.0 & 0.0 & 0.0 & 0.0 & 0.0 & 0.0 & 0.0 & 1.0 & 0.0 \\
    0.0 & 1.0 & 0.0 & 0.0 & 0.0 & 0.0 & 0.0 & 0.0 & 1.0 & 0.0 \\
  \end{pmatrix}\T
  $} \,,
\end{equation}
the PSR approximation (\Cref{thm:psr.rewards.approx}) $R\psr = U\PI R\pomdp$ ,
\begin{equation}
  R\psr = \begin{pmatrix}
    0.5 & -0.5 & -0.5 & 0.5 & 0.5 \\
    0.5 & -0.5 & -0.5 & 0.5 & 0.5 \\
  \end{pmatrix}\T \,,
\end{equation}
and the POMDP reconstruction (\Cref{thm:psr.reconstruction}) $\tilde R\pomdp =
U R\psr$ (note that states $0$, $1$, $8$ and $9$ emit rewards),
\begin{equation}
  \resizebox{.89\linewidth}{!}{$
  \displaystyle
  \tilde R\pomdp = \begin{pmatrix}
    0.5 & 0.5 & 0.0 & 0.0 & 0.0 & 0.0 & 0.0 & 0.0 & 0.5 & 0.5 \\
    0.5 & 0.5 & 0.0 & 0.0 & 0.0 & 0.0 & 0.0 & 0.0 & 0.5 & 0.5 \\
  \end{pmatrix}\T
  $} \,.
\end{equation}

The approximate rewards $R\psr$ are hard to interpret, and it is not obvious
that something is wrong.  However, looking at the reconstructed rewards $\tilde
R\pomdp$ (also shown in parentheses in \Cref{fig:loadunload}), we notice that
the first two and the last two columns---respectively corresponding to states
$0$ and $1$, and $8$ and $9$---have the same values, indicating that the PSR
has lost the ability to distinguish between those states.  After all, the only
reason why states $0$ and $1$ are separate states in load/unload is because they
lead to different rewards;  they are equivalent in all other aspects, hence the
PSR ``efficiently'' merges them.
More critically, the optimal behavior prescribed by $\tilde R\pomdp$ (read
$R\psr$) has changed catastrophically from that of $R\pomdp$.  With $R\pomdp$,
the optimal behavior is to move back-and-forth between the corridor ends; with
$\tilde R\pomdp$ (read $R\psr$), the optimal behavior is to reach one end and
stay there.

\section{Reward-Predictive State Representations}

In this section, we introduce \emph{reward-predictive state representations}
(R-PSRs), a generalization of PSRs which accurately models both the observation
processes \emph{and} the reward processes of all finite POMDPs.  The derivation
of R-PSRs resembles closely that of PSRs, with a few key differences:
\begin{enumerate*}[label=(\alph*)]
  \item a token action is used to unify observation and reward emissions,
  \item tests are extended by a final action used which switches between
    observation and reward emissions,
  \item test probabilities are generalized to include reward information.
\end{enumerate*}

\paragraph{Extended Action-Space}
We define an extended action-space $\zset\doteq\aset\cup\{\zeta\}$, where
$\zeta$ is a token action, and extend the reward function such that $R(s,
\zeta) = 1$ for every state $s$.  Token action $\zeta$ (with its reward) is a
construct which will allow us to define a single model of both observation and
reward emissions;  however, we are not changing the space of actions available
to the agent, and $\zeta$ cannot be used to interact with the environment, nor
can it be part of a history or test.

\paragraph{Intents and their Expectations}
In R-PSRs, the system state incorporates the expected rewards of hypothetical
futures we call \emph{intents}.
We define an \emph{intent} $qz$ as a test $q$ followed by an extended action
$z$, and the space of all intents as $\iset\doteq \qset\times\zset$.  An
\emph{intent reward} $r(qz\mid h) \doteq \Pr(q\mid h) R(hq, z)$ is the test
probability multiplied by the expected reward obtained by the intent action
following the concatenated history-test.  
Intent rewards are the core quantity modeled by R-PSRs\@, and a \emph{linear}
reward-predictive state $r(h) \in\realset^{|\iset\core|}$ (where $\iset\core$
is a \emph{core} set of intents, defined later) is a representation of history
$h$ such that intent rewards are linear in $r(h)$, i.e., $r(qz\mid h) = r(h)\T
m_{qz}$, where $m_{qz} \in\realset^{|\iset\core|}$ is the \emph{parameter
vector} associated to intent $qz$.
A focal property of $r(qz\mid h)$, provided by the token action $\zeta$, is
that it generalizes both test probabilities and history rewards,
\begin{align}
  R(h, \zeta) &= 1 &\implies &&r(q\zeta\mid h) &= \Pr(q\mid h) \,,
  \label{eq:rpsr:o} \\
  \Pr(\notest \mid h) &= 1 &\implies &&r(\notest a\mid h) &= R(h, a) \,.
  \label{eq:rpsr:r}
\end{align}
The intent reward function can be expressed in a vectorized form similar to
that of PSRs (proof in Appendix),
\begin{equation}
  r(qz\mid h) = b(h)\T G_q\T \left[ R\pomdp \right]_{:z} \,.
  \label{eq:rpsr:matrix.form}
\end{equation}

\paragraph{Outcome Vectors}
The \emph{outcome} of an intent $u(qz)\in \realset^{|\sset|}$ is a vector
indicating the intent rewards from each state as starting point, i.e., $\left[
u(qz) \right]_i = \Exp\left[ \rfn(\bm{s'}, a) \mid s=i, q \right]$.  Outcomes
can be defined recursively using the generative matrix,
\begin{align}
  u(\notest a) &= \left[ R\pomdp \right]_{:a} \,, \label{eq:rpsr:outcome.init} \\
  u(aoqz) &= G_{ao}\T u(qz) \,. \label{eq:rpsr:outcome.recursive}
\end{align}
\noindent Combining
\Cref{eq:rpsr:matrix.form,eq:rpsr:outcome.init,eq:rpsr:outcome.recursive}
results in $r(qz\mid h) = b(h)\T u(qz)$, i.e., the intent reward given a
history is the expectation of intent rewards given each state.
A set of intents is said to be \emph{linearly independent} iff the respective
outcome vectors are linearly independent, and any maximal set of linearly
independent intents is called a \emph{core set}, denoted as $\iset\core$.  R-PSR
core sets share similar properties to PSR core sets:  there are infinite core
sets which share the same size $|\iset\core|$, called the R-PSR rank, which is
upper-bounded by $|\iset\core|\leq |\sset|$.  Note that the PSR rank and the
R-PSR rank are not necessarily equal.
\Cref{alg:search:intent:dfs} is a depth-first search algorithm to find a core
set, analogous to the original PSR search
algorithm~\cite{littman_predictive_2002}, but adapted for R-PSRs (see
Appendix for a more efficient breadth-first variant).

\begin{algorithm}[tb]
  \caption{Depth-first search of a maximal set of linearly independent intents
  $\iset\core$.}\label{alg:search:intent:dfs}
  \begin{algorithmic}
    \Require $q\in\qset$ (optional, default $\notest$)
    \Require Independent intents $I\subset \iset$ (optional, default
    $\varnothing$)
    \Ensure Maximal set of independent intents which either belong to $I$ or
    are extensions of $q$.
    \Function{ISearchDFS}{$q, I$}
    \ForAll{$z\in\zset$}
    \If{$u(qz)$ independent of $\{u(i)\mid i\in I\}$}
    \ForAll{$a\in\aset$, $o\in\oset$}
    \State $I \gets \Call{ISearchDFS}{aoq, I \cup \{ qz \}}$
    \EndFor
    \EndIf
    \EndFor
    \State \Return $I$
    \EndFunction
  \end{algorithmic}
\end{algorithm}

\paragraph{Reward-Predictive States}
The outcome matrix $U\in\realset^{|\sset|\times|\iset\core|}$ of a core set
$\iset\core$ is the column-wise stacking of the core outcome vectors $\{ u(qz)
\mid qz\in\iset\core \}$.
By the definition of a core set, the outcome $u(qz)$ of any intent $qz$ is a
linear combination of the core outcome matrix $U$ columns (else the core set
would not be maximal), i.e., $u(qz) \in \colspace U$ and, because $UU\PI$ is
the projection onto $\colspace U$, then $u(qz) = UU\PI u(qz)$.  Consequently,
$r(qz\mid h) = b(h) \T u(qz) = b(h)\T UU\PI u(qz) = r(h)\T m_{qz}$ where
$r(h)\T \doteq b(h)\T U$ is the reward-predictive state, and $m_{qz} \doteq
U\PI u(qz)$ is the parameter vector of $qz$.  Each dimension of $r(h)$ is
itself the intent reward of a core intent, i.e., $\left[ r(h) \right]_i = r(qz
\mid h)$, where $qz$ is the $i^{th}$ core intent.

\paragraph{Emissions and Dynamics}
According to \Cref{eq:rpsr:o,eq:rpsr:r}, observation probabilities and reward
emissions are $\Pr(o\mid h, a) = r(ao\zeta\mid h) = r(h)\T m_{ao\zeta}$ and
$R(h, a) = r(\notest a\mid h) = r(h)\T m_{\notest a}$ respectively.
Therefore, the test-less intent parameters $\{ m_{\notest a} \mid a\in\aset \}$
constitute the columns of reward matrix $R\rpsr$.
The reward-predictive state dynamics are $r(hao) = \sfrac{ \left( r(h)\T M_{ao}
\right) }{ \left( r(h)\T m_{ao\zeta} \right) }$, where $M_{ao}$ is the
column-wise stacking of the extended core intent parameters $\{ m_{aoqz} \mid
qz\in\iset\core \}$
%
(proof in Appendix).

%



\paragraph{Value Iteration for R-PSRs}
Value iteration can be adapted to R-PSRs (R-PSR-VI).  Both the derivation and
the equation to compute the alpha vectors $\alpha\rpsr_\pi$ are identical to
those of PSR-VI~\cite{james_planning_2004,boots_closing_2011}, also shown in
\Cref{sec:vi};  the only difference being that parameters $R\rpsr$ and
$M\rpsr_{ao}$ are used, whose values correctly represent the POMDP/R-PSR
rewards.

\section{Evaluation}\label{sec:evaluation}

We perform empirical evaluations to confirm the theory developed in this work,
the issues with PSRs, and the validity of R-PSRs.  We show that
\begin{enumerate*}[label=(\alph*)]
  \item a non-trivial portion of finite POMDPs used in classical literature do
    not satisfy the accuracy condition of \Cref{thm:psr.rewards}, i.e., this is
    a common problem occurring in common domains;
  \item PSR-VI based on inaccurate approximate PSR rewards results in
    catastrophically sub-optimal policies, i.e., approximate rewards are not
    viable for control;
  \item R-PSRs are accurate reward models; and
  \item R-PSR-VI results in the same optimal policies as POMDP-VI\@.
\end{enumerate*}
%
%
Our evaluation involves a total of $63$ unique domains: $60$ are taken from
Cassandra's POMDP page~\cite{cassandra_pomdp_1999}, a repository of classic
finite POMDPs from the literature; $2$ are the well-known
\emph{load/unload}~\cite{meuleau_learning_1999} and
\emph{heaven/hell}~\cite{bonet_solving_1998}; and the last one is
\emph{float/reset}~\cite{littman_predictive_2002}.

\begin{table*}[tb]
  \centering
  \begin{tabular}{ccccccccc}
    \toprule
    & \emph{4x3}& \emph{heaven/hell} & \emph{iff} & \emph{line4-2goals} & \emph{load/unload} & \emph{paint} & \emph{parr} & \emph{stand-tiger} \\
    \midrule
    \phantom{rel-}$d_\infty$ & 1.0 & 1.0 & 48.93 & 0.6\phantom{0} & 0.5 & 1.33 & 1.0 &
    65.0\phantom{0} \\
    rel-$d_\infty$ & 1.0 & 1.0 & \phantom{0}0.75 &
    0.75 & 0.5 & 1.33 & 0.5 & \phantom{0}0.65 \\
    \bottomrule
  \end{tabular}
  \caption{PSR reward errors. Measure $d_\infty \doteq \|R\pomdp - \tilde
    R\pomdp\|_\infty$ is the $\ell_\infty$ distance between the POMDP rewards
    and their reconstruction.  Relative measure rel-$d_\infty \doteq \sfrac{
    d_\infty }{ \| R\pomdp \|_\infty }$ is normalized w.r.t.\ the scale of
  POMDP rewards.  The R-PSR reward errors (omitted) are all
zero.}\label{tab:psr:errors}
\end{table*}

\begin{table*}[tb]
\centering 
\begin{tabular}{llS[table-format=+2.1(1)]S[table-format=+2.1(1)]S[table-format=+2.1(1)]S[table-format=+2.1(1)]}
\toprule
   {Domain} & {Model} &      {Random} &       {POMDP-VI} &         {PSR-VI} &       {R-PSR-VI} \\
\midrule
\multirow{2}{*}{\emph{heaven/hell}} &   POMDP/R-PSR &   0.0 \pm 0.1 &   \bfseries 1.4 \pm 0.0 &   0.0 \pm 0.0 &   \bfseries 1.4 \pm 0.0 \\
&     PSR &  \bfseries -0.0 \pm 0.0 &  \bfseries -0.0 \pm 0.0 &  \bfseries -0.0 \pm 0.0 &  \bfseries -0.0 \pm 0.0 \\
\midrule
\multirow{2}{*}{\emph{line4}-2goals} &   POMDP/R-PSR &  \bfseries 0.4 \pm 0.0 &  \bfseries 0.4 \pm 0.0 &  \bfseries 0.4 \pm 0.0 &  \bfseries 0.4 \pm 0.0 \\
&     PSR &  \bfseries 4.0 \pm 0.0 &  \bfseries 4.0 \pm 0.0 &  \bfseries 4.0 \pm 0.0 &  \bfseries 4.0 \pm 0.0 \\
\midrule
\multirow{2}{*}{\emph{load/unload}} &   POMDP/R-PSR &  1.2 \pm 0.5 &  \bfseries 4.5 \pm 0.1 &  0.6 \pm 0.2 &  \bfseries 4.5 \pm 0.1 \\
&     PSR &  4.0 \pm 1.0 &  2.6 \pm 0.1 &  \bfseries 9.1 \pm 0.5 &  2.6 \pm 0.1 \\
\midrule
\multirow{2}{*}{\emph{paint}} &   POMDP/R-PSR &  -4.2 \pm 1.4 &  \bfseries 3.3 \pm 0.3 &  0.0 \pm 0.0 &  \bfseries 3.3 \pm 0.3 \\
&     PSR &  -3.2 \pm 1.0 &  1.0 \pm 0.9 &  \bfseries 3.3 \pm 0.0 &  1.0 \pm 1.0 \\
\midrule
\multirow{2}{*}{\emph{parr}} &   POMDP/R-PSR &  4.3 \pm 1.7 &  \bfseries 7.1 \pm 0.0 &  6.5 \pm 1.8 &  \bfseries 7.1 \pm 0.0 \\
&     PSR &  4.3 \pm 0.8 &  3.6 \pm 0.0 &  \bfseries 6.3 \pm 0.0 &  3.6 \pm 0.0 \\
\midrule
\multirow{2}{*}{\emph{stand-tiger}} &   POMDP/R-PSR &  -122.3 \pm 43.1 &    \bfseries 49.2 \pm 23.4 &  0.0 \pm 0.0 &    \bfseries 49.8 \pm 23.2 \\
&     PSR &  -122.7 \pm 26.4 &  -151.1 \pm 17.6 &  \bfseries 0.0 \pm 0.0 &  -150.2 \pm 18.0 \\
\bottomrule
\end{tabular}
\caption{Policy return estimates for each policy (columns) by each
  model (rows), where the identical POMDP and R-PSR rows were merged.  Means
and standard deviations shown as $\mu\pm\sigma$.  Bold text indicates, for each
model, the highest performing policy.}
\label{tab:results}
\end{table*}

%
We found that $8$ out of these $63$ domains---a non-trivial amount---do not
satisfy the accuracy condition.  \Cref{tab:psr:errors} shows the reward errors
between original and reconstructed POMDP rewards.  We note that all domains
where accurate PSR rewards are not possible also have high relative errors,
which implies that inaccurate PSRs are unlikely to be only \emph{mildly }
inaccurate.
VI is a fairly expensive algorithm which does not scale well with long planning
horizons and is thus not suitable for all problems; convergence to a steady
optimal value function was possible within a reasonable time-frame for only $6$
of the $8$ domains.  For each of these $6$ domains, we run $4$ different policies for
$1000$ episodes of $100$ steps: the uniform policy, and the policies
respectively obtained by POMDP-VI, PSR-VI, and R-PSR-VI\@.  Every
action-observation sequence is then evaluated by the POMDP, PSR, and R-PSR
reward models.  \Cref{tab:results} shows the return estimates for POMDPs, PSRs,
and R-PSRs;  Note that the POMDP and R-PSR rows are combined since they contain
the same values.

In this context, the POMDP represents the true task, which the PSR and R-PSR
also attempt to encode;  hence, the POMDP rows show how well each model's
respective policy solves the original task, while the PSR rows show how the
PSR's (inaccurate) encoded task evaluates each model's respective policy.
Notably, POMDPs and R-PSRs consistently agree with high numerical precision on
the return values of all trajectories, whereas PSRs consistently disagree.  The
POMDP-VI and R-PSR-VI policies achieve the same values throughout all
experiments, i.e., they consistently converge to the same policies.  Since the
POMDP encodes the true task, both POMDP-VI and R-PSR-VI represent the optimal
policy which solves that task.
With the singular exception of \emph{line4-2goals}, a trend appears where the
PSR-VI policy is sub-optimal and, in the case of \emph{load/unload}, is even
worse than the random policy (see \Cref{sec:casestudy}).  Vice versa, the
POMDP-VI/R-PSR-VI policies are sub-optimal according to the PSR model; notably,
the random policy performs better than the POMDP-VI/R-PSR-VI policies in 3 out
of 6 cases, which underlines just how much the task encoded by the PSRs have
diverged from their original form.

These results reaffirm not only the theory developed in this document, i.e.,
that PSRs are equivalent to POMDPs only in relation to their observation
process and not their reward process, but also that this is a common problem
which causes significant control issues.  Further, the results confirm the
validity of the developed R-PSR theory, and the equivalence between POMDPs and
R-PSRs.  Overall, this confirms that R-PSRs are better suited for control,
compared to vanilla PSRs.

\section{Conclusions}

In this work, we presented theoretical results on the accuracy of PSR rewards
relative to POMDP rewards, identified a sufficient and necessary condition
which determines whether a PSR can accurately represent POMDP rewards, and
derived the closest linear approximate rewards for non-accurate PSRs.  We also
showed empirically that reward approximations are likely to warp the implied
task in undesirable ways.  Therefore, we proposed R-PSRs, a generalization of
PSRs which encodes reward values jointly with test probabilities, and solves
the reward modeling problem of vanilla PSRs while remaining faithful to the
idea of grounding a system state representation on non-latent quantities, and
avoiding the pitfalls of observable rewards.
R-PSRs combine the benefits of both POMDPs and PSRs:  Compared to POMDPs,
R-PSRs do not rely on a latent state, which makes model learning easier and
grounded in non-latent quantities;  Compared to PSRs, R-PSRs are able to model
a wider range of tasks, and the reward structure of any finite POMDP\@.
In future work, we aim to adapt more learning and planning algorithms to
R-PSRs, and address the \emph{discovery} problem, i.e., the problem of learning
a core set of intents from sample interactions.
%

\section*{Acknowledgments}

This research was funded by NSF award 1816382.

\clearpage

\bibliographystyle{named}
\bibliography{references}

\clearpage
\appendix
\section{Lemmas, Theorems, and Proofs}\label{app:theorems}

For convenience, we repeat all the theorems already stated in the main
document before their proofs.

\subsection{PSR Theorems}\label{app:theorems:psr}

{
\renewcommand{\theproposition}{\ref{thm:psr.failure}}
\begin{proposition}
  %
  %
  For any finite POMDP and its respective PSR, a (linear or non-linear)
  function $f(p(h), a) \mapsto R\pomdp(h, a)$ does not necessarily exist.
\end{proposition}
\addtocounter{proposition}{-1}
}

\begin{proof}[Proof by example]
  Consider a POMDP with a large state-space $|\sset|\gg 1$, a large
  action-space $|\aset| \gg 1$, but a singleton observation-space $|\oset|=1$.
  Because there is only one observation, every test probability is $p(q\mid h)
  = 1$, and any singleton test set is a core set $\qset\core = \{ q \}$.  The
  PSR state $p(h)$ is a $1$-dimensional unitary vector which is stationary with
  respect to the history; consequently, any hypothetical PSR reward function
  $f(p(h), a)$ is also stationary.  In contrast, the belief state $b(h)$ is not
  necessarily stationary, and neither is the POMDP reward function $R(h, a)$.
  Therefore, a PSR reward function $f(p(h), a)$ equivalent to a POMDP reward
  function $R\pomdp(h, a)$ does not necessarily exist.
\end{proof}

{
\renewcommand{\thetheorem}{\ref{thm:psr.rewards}}
\begin{theorem}[Accurate Linear PSR Rewards]
  %
  %
  A POMDP reward matrix $R\pomdp$ can be accurately converted to a PSR reward
  matrix $R\psr$ iff every column of $R\pomdp$ is linearly dependent on the
  core outcome vectors (the columns of $U$).  If this condition is satisfied,
  we say that the PSR is \emph{accurate}, and $R\psr = U\PI R\pomdp$
  accurately represents the POMDP rewards.
\end{theorem}
\addtocounter{theorem}{-1}
}

\begin{proof}
  Consider a linear PSR reward matrix $R\psr$.  In a POMDP, $b(h)\T R\pomdp
  \in\realset^{|\aset|}$ is the vector of expected rewards following history
  $h$.  In a PSR, $p(h)\T R\psr\in\realset^{|\aset|}$ represents the same
  quantity.  The PSR rewards are equivalent to the POMDP rewards iff the two
  vectors are equal for every possible history.  Consider the reward error
  \begin{align}
    \epsilon(h) &\doteq \frac{1}{2} \left\| p(h)\T R\psr - b(h)\T R\pomdp
    \right\|^2 \nonumber \\
    &= \frac{1}{2} \left\| b(h)\T U R\psr  - b(h)\T R\pomdp  \right\|^2
    \nonumber \\
    &= \frac{1}{2} \left\| b(h)\T \left( U R\psr  - R\pomdp  \right) \right\|^2
  \end{align}

  Assume that the POMDP is non-degenerate, in that every state is reachable (if
  not every state is reachable, then take into account the non-degenerate POMDP
  obtained by ignoring those states).  Then, every dimension of $b(h)$ is
  strictly positive for some history $h$, and the error $\epsilon(h)$ is
  guaranteed to be zero for every history iff $U R\psr = R\pomdp$, i.e., the
  columns of $R\pomdp$ are linear combinationf of the columns of $U$.  Because
  $U$ is full column rank, $U\PI U$ is the identity matrix, and
  \begin{align}
    U R\psr &= R\pomdp \\
    U\PI U R\psr &= U\PI R\pomdp \\
    R\psr &= U\PI R\pomdp
  \end{align}
\end{proof}

{
\renewcommand{\thecorollary}{\ref{thm:pomdp.rewards}}
\begin{corollary}
  %
  %
  Assuming that a PSR can be represented by a finite POMDP to begin with, then
  any PSR rewards $R\psr$ are accurately represented by POMDP rewards $R\pomdp
  = UR\psr$.
\end{corollary}
\addtocounter{corollary}{-1}
}

\begin{proof}
  $R\pomdp = U R\psr$ satisfies the accuracy condition of
  \Cref{thm:psr.rewards}, therefore we can reconvert it back to a PSR reward
  matrix via $\hat R\psr = U\PI R\pomdp$.  Because $U$ is full column rank,
  $U\PI U$ is the identity matrix, and the rountrip conversion $R\psr \mapsto
  R\pomdp \mapsto \hat R\psr$ is always consistent, i.e.,
  \begin{align}
    \hat R\psr &= U\PI R\pomdp \nonumber \\
    &= U\PI U R\psr \nonumber \\
    &= R\psr \,.
  \end{align}
  Because the POMDP-to-PSR reward conversion is accurate, and the roundtrip
  conversion is also accurate, then the PSR-to-POMDP reward conversion must
  also be accurate.
\end{proof}

{
\renewcommand{\thetheorem}{\ref{thm:psr.rewards.approx}}
\begin{theorem}
  %
  %
  The linear approximation of POMDP rewards for non-accurate PSRs which results
  in the lowest reward approximation error is $R\psr \doteq U\PI R\pomdp$.
\end{theorem}
\addtocounter{theorem}{-1}
}

\begin{proof}
  Following \Cref{thm:psr.rewards}, we try to find the PSR rewards $R\psr$ such
  that $UR\psr$ is as close as possible to $R\pomdp$ in a least-squares fashion.
  We consider the rewards associated with each action $a\in\aset$ in isolation,
  i.e., the columns $\left[ R\psr \right]_{:a}$ and $\left[ R\pomdp
  \right]_{:a}$, and the respective reward error vector
  \begin{align}
    \epsilon_a &\doteq \frac{1}{2} \left\| U \left[ R\psr \right]_{:a} - \left[ R\pomdp
    \right]_{:a} \right\|^2\, . \\
    \intertext{Because $\epsilon_a$ is convex in $\left[ R\psr \right]_{:a}$,
    its minimum corresponds to the unique stationary point,}
    & \nabla \frac{1}{2} \left\| U \left[ R\psr \right]_{:a} - \left[ R\pomdp
    \right]_{:a} \right\|^2  \nonumber \\
    &= U\T \left( U \left[ R\psr \right]_{:a} - \left[ R\pomdp \right]_{:a}
    \right) \nonumber \\
    &\stackrel{!}{=} \bar 0 \\
    \intertext{which results in}
    U\T U \left[ R\psr \right]_{:a} &= U\T \left[ R\pomdp \right]_{:a} \\
    \left[ R\psr \right]_{:a} &= \left( U\T U \right)\I U\T \left[ R\pomdp
    \right]_{:a} \nonumber \\
    &= U\PI \left[ R\pomdp \right]_{:a} \,.
  \end{align}
  Stacking the optimal vectors for each action column-wise, we obtain the
  optimal reward matrix $R\psr = U\PI R\pomdp$.
\end{proof}

{
\renewcommand{\thecorollary}{\ref{thm:psr.reconstruction}}
\begin{corollary}
  %
  %
  $\tilde R\pomdp \doteq UU\PI R\pomdp$ is the reconstructed POMDP-form of the
  PSR approximation $R\psr$ of the true POMDP rewards $R\pomdp$.  $\tilde
  R\pomdp$ and $R\pomdp$ are equal iff the accuracy condition is satisfied.
\end{corollary}
\addtocounter{corollary}{-1}
}

\begin{proof}
  Follows from \Cref{thm:psr.rewards.approx,thm:pomdp.rewards}.
\end{proof}

\subsection{R-PSR Theorems}\label{app:theorems:rpsr}

\begin{lemma}\label{thm:generative}
  \begin{equation}
    G_q b(h) = \Pr(q\mid h) b(hq) \,.
  \end{equation}
\end{lemma}

\begin{proof}
  First we note that, by the definition of the generative matrices, $\left[
  G_{q} \right]_{ij} = \Pr(s'=i, o_q\mid s=j, a_q)$, where $s$ is the state at
  the start of test $q$, and $s'$ is the state at the end of test $q$.
  Consequently,
  \begin{align}
    \left[ G_q b(h) \right]_i &= \sum_j \Pr(s'=i, o_q\mid s=j, a_q) \Pr(s=j\mid h) \nonumber \\
                              &= \Pr(s'=i, o_q\mid h, a_q) \nonumber \\
                              &= \Pr(o_q\mid h, a_q) \Pr(s'=i\mid h, a_q, o_q) \nonumber \\
                              &= \Pr(q\mid h) \left[ b(hq) \right]_i
  \end{align}
\end{proof}

\begin{proposition}[Vectorized Form of the Intent Rewards Function]\label{thm:rpsr}
  \begin{equation}
    r(qz\mid h) = b(h)\T G_q\T \left[ R\pomdp \right]_{:z}
  \end{equation}
\end{proposition}

\begin{proof}
  Following \Cref{thm:generative},
  \begin{align}
    r(qz\mid h) &= \Pr(q\mid h) R(hq, z) \nonumber \\
                &= \Pr(q\mid h) b(hq)\T \left[ R\pomdp \right]_{:z} \nonumber \\
                &= b(h)\T G_q\T \left[ R\pomdp \right]_{:z}
  \end{align}
\end{proof}

\begin{proposition}[Vectorized Form of R-PSR Dynamics]\label{thm:rpsr:dynamics}
  %
  %
  In vectorized form, the reward-predictive state dynamics are $r(hao) = \sfrac{
  \left( r(h)\T M_{ao} \right) }{ \left( p(h)\T m_{ao\zeta} \right) }$, where
  $M_{ao}$ is the column-wise stacking of the extended core intent parameters $\{
  m_{aoqz} \mid qz\in\iset\core \}$.
\end{proposition}

\begin{proof}
  The $i^{th}$ dimension of the updated predictive-reward state $r(hao)$
  (corresponding to the $i^{th}$ core intent $qz\in\iset\core$) is
  \begin{align}
    \left[ r(hao) \right]_i &= r(qz\mid hao) \nonumber \\
                &= b(hao)\T G_q\T \left[ R\pomdp \right]_{:z} \nonumber \\
                &= \frac{ b(h)\T G_{ao}\T }{ \Pr(o\mid h, a) } G_q\T \left[
                R\pomdp \right]_{:z} \nonumber \\
                &= \frac{ b(h)\T G_{aoq}\T \left[ R\pomdp \right]_{:z} }{
                \Pr(o\mid h, a) } \nonumber \\
                &= \frac{ r(aoqz\mid h) }{ r(ao\zeta\mid h) } \nonumber \\
                &= \frac{ r(h)\T m\rpsr_{aoqz} }{ r(h)\T m\rpsr_{ao\zeta} }
  \end{align}

  In vectorized form, the reward-predictive state dynamics are
  \begin{equation}
    r(hao) = \frac{ r(h)\T M_{ao} }{ r(h)\T m_{ao\zeta} }\,,
  \end{equation}
  \noindent where $M_{ao}$ is the column-wise stacking of the extended core
  intent parameters $\{ m_{aoqz} \mid qz\in\iset\core \}$.
\end{proof}

\section{Algorithms}\label{app:algo}

\Cref{alg:search:intent:bfs} is the breadth-first variant of the depth-first
core intent search algorithm (\Cref{alg:search:intent:dfs}), which finds shorter
core intents.

\begin{algorithm}[tb]
  \caption{Breadth-first search of a maximal set of linearly independent
  intents $\iset\core$.}\label{alg:search:intent:bfs}
  \begin{algorithmic}
    \Ensure Core intent set $\iset\core$
    \Function{ISearchBFS}{}
    \State $I \gets \varnothing$
    \ForAll{$z\in\zset$}
    \If{$u(\notest z)$ independent of $\{u(i)\mid i\in I\}$}
        \State $I \gets I \cup \left\{ \notest z \right\}$
      \EndIf
    \EndFor
    \Repeat
      \State $I' \gets I$
      \ForAll{$a\in\aset$, $o\in\oset$, $qz \in I$}
        \If{$u(aoqz)$ independent of $\{u(i)\mid i\in I\}$}
          \State $I \gets I \cup \left\{ aoqz \right\}$
        \EndIf
      \EndFor
    \Until{$I' = I$}
    \State \Return $I$
    \EndFunction
  \end{algorithmic}
\end{algorithm}

\onecolumn

\section{The Load/Unload POMDP}\label{app:loadunload}
\begin{verbatim}
# This is the POMDP source for the Load/Unload problem used by
# Nicolas Meuleau, Peshkin and Kaelbling. This is a simple POMDP
# of 10 states as follows: A road is split into 5 segments. At
# the left end we pick up an item from an infinite pile. At the right,
# we drop our item. A reward of 1 is received for picking up or putting
# down, but only one item can be carried at a time.
# Observations are: 'loading' when in the leftmost state
#                   'unloading' in the rightmost state
#                   'travelling' otherwise
# To act optimally the belief state must encode whether we have an item or
# not. This is an interesting problem because we require only 1 bit
# of memory to act optially, so we can use efficient techniques that do
# not attempt to determine a distribution over states. Actions are simply
# move_left or move_right. Moving in an impossible direction leaves the
# agent where it is. This is also a nice problem because we can easily
# increase the mixing time by increasing the length of the road.
# Doug Aberdeen, 2001
discount: 0.95
values: reward
states: 10
actions:  right left
observations: loading unloading travel

start: uniform
# Even columns (starting from 0) are loaded
# Odd columns are unloaded

#Load in leftmost

T : left
1.0 0.0 0.0 0.0 0.0 0.0 0.0 0.0 0.0 0.0
1.0 0.0 0.0 0.0 0.0 0.0 0.0 0.0 0.0 0.0
1.0 0.0 0.0 0.0 0.0 0.0 0.0 0.0 0.0 0.0
0.0 1.0 0.0 0.0 0.0 0.0 0.0 0.0 0.0 0.0
0.0 0.0 1.0 0.0 0.0 0.0 0.0 0.0 0.0 0.0
0.0 0.0 0.0 1.0 0.0 0.0 0.0 0.0 0.0 0.0
0.0 0.0 0.0 0.0 1.0 0.0 0.0 0.0 0.0 0.0
0.0 0.0 0.0 0.0 0.0 1.0 0.0 0.0 0.0 0.0
0.0 0.0 0.0 0.0 0.0 0.0 0.0 1.0 0.0 0.0
0.0 0.0 0.0 0.0 0.0 0.0 0.0 1.0 0.0 0.0 
#L0  U0  L1  U1  L2  U2  L3  U3  L4  U4
# 0   1   2   3   4   5   6   7   8   9

T : right
0.0 0.0 1.0 0.0 0.0 0.0 0.0 0.0 0.0 0.0
0.0 0.0 1.0 0.0 0.0 0.0 0.0 0.0 0.0 0.0
0.0 0.0 0.0 0.0 1.0 0.0 0.0 0.0 0.0 0.0
0.0 0.0 0.0 0.0 0.0 1.0 0.0 0.0 0.0 0.0
0.0 0.0 0.0 0.0 0.0 0.0 1.0 0.0 0.0 0.0
0.0 0.0 0.0 0.0 0.0 0.0 0.0 1.0 0.0 0.0
0.0 0.0 0.0 0.0 0.0 0.0 0.0 0.0 1.0 0.0
0.0 0.0 0.0 0.0 0.0 0.0 0.0 0.0 0.0 1.0
0.0 0.0 0.0 0.0 0.0 0.0 0.0 0.0 0.0 1.0
0.0 0.0 0.0 0.0 0.0 0.0 0.0 0.0 0.0 1.0 

O : *
1.0 0.0 0.0
1.0 0.0 0.0
0.0 0.0 1.0
0.0 0.0 1.0
0.0 0.0 1.0
0.0 0.0 1.0
0.0 0.0 1.0
0.0 0.0 1.0
0.0 1.0 0.0
0.0 1.0 0.0

R : * : 1 : * : * 1.0
R : * : 8 : * : * 1.0
\end{verbatim}

\twocolumn

\end{document}